\documentclass[english,a4paper,12pt]{article}

\usepackage{amsxtra, amsfonts, amssymb, amstext, amsmath, mathtools}
\usepackage{amsthm}
\usepackage{booktabs}
\usepackage{nicefrac}
\usepackage{xspace}
\usepackage{url}
\usepackage{graphics,color}
\usepackage[algo2e,ruled,vlined,linesnumbered]{algorithm2e}
\usepackage{wrapfig}

\newtheorem{theorem}{Theorem}

\newtheorem{proposition}[theorem]{Proposition}

\newcommand{\oea}{\mbox{$(1 + 1)$~EA}\xspace}

\newcommand{\oplea}{\mbox{$(1+\lambda)$~EA}\xspace}

\newcommand{\oclea}{\mbox{$(1,\lambda)$~EA}\xspace}
\newcommand{\opllga}{\mbox{$(1+(\lambda,\lambda))$~GA}\xspace}

\newcommand{\OM}{\textsc{OM}\xspace}
\newcommand{\onemax}{\textsc{OneMax}\xspace}
\newcommand{\LO}{\textsc{Leading\-Ones}\xspace}
\newcommand{\leadingones}{\LO}
\newcommand{\needle}{\textsc{Needle}\xspace}

\newcommand{\plateau}{\textsc{Plateau}\xspace}
\newcommand{\jump}{\textsc{Jump}\xspace}

\DeclareMathOperator{\poly}{poly}

\newcommand{\R}{\ensuremath{\mathbb{R}}}

\newcommand{\N}{\ensuremath{\mathbb{N}}} 
\newcommand{\Z}{\ensuremath{\mathbb{Z}}}

\newcommand{\calE}{\ensuremath{\mathcal{E}}} 
 
\newcommand{\calP}{\ensuremath{\mathcal{P}}}

\DeclareMathOperator{\Geom}{Geom}

\newcommand{\eps}{\varepsilon}

\begin{document}
\title{Exponential Upper Bounds for the Runtime of Randomized Search Heuristics\thanks{Extended version of a paper appearing in the proceedings of PPSN 2020~\cite{Doerr20ppsnUB}. This version contains all proofs and other details that had to be omitted in the conference version for reasons of space (roughly an additional 50\% text over the conference version) as well as the new Section~\ref{sec:simpleGA} on the simple genetic algorithm.}}

\author{Benjamin Doerr\\ Laboratoire d'Informatique (LIX)\\ CNRS, \'Ecole Polytechnique\\ Institut Polytechnique de Paris\\ Palaiseau\\ France
}

\maketitle

{\sloppy

\begin{abstract}
  We argue that proven exponential upper bounds on runtimes, an established area in classic algorithms, are interesting also in heuristic search and we prove several such results. We show that any of the algorithms randomized local search, Metropolis algorithm, simulated annealing, and (1+1) evolutionary algorithm can optimize any pseudo-Boolean weakly monotonic function under a large set of noise assumptions in a runtime that is at most exponential in the problem dimension~$n$. This drastically extends a previous such result, limited to the (1+1) EA, the LeadingOnes function, and one-bit or bit-wise prior noise with noise probability at most $1/2$, and at the same time simplifies its proof. With the same general argument, among others, we also derive a sub-exponential upper bound for the runtime of the $(1,\lambda)$ evolutionary algorithm on the OneMax problem when the offspring population size $\lambda$ is logarithmic, but below the efficiency threshold. To show that our approach can also deal with non-trivial parent population sizes, we prove an exponential upper bound for the runtime of the mutation-based version of the simple genetic algorithm on the OneMax benchmark, matching a known exponential lower bound.
\end{abstract}

\section{Introduction}

The mathematical analysis of runtimes of randomized search heuristics is an established field of the general area of heuristic search~\cite{NeumannW10,AugerD11,Jansen13,DoerrN20}. The vast majority of the results in this area show that a certain algorithm can solve (or approximately solve) a certain problem within some polynomial runtime (polynomial upper bound on the runtime) or show that this is not possible by giving a super-polynomial, often exponential, lower bound on the runtime. 

As a rare exception to this rule, in his extensive analysis of how the $(1+1)$ evolutionary algorithm (\oea)\footnote{See Section~\ref{sec:prelim} for details on all technical terms used in this introduction.} optimizes the \leadingones benchmark in the presence of prior noise, Sudholt~\cite[Theorem~6]{Sudholt20} showed that for one-bit or bit-wise noise with noise probability at most $\frac 12$, the \oea finds the optimum of \leadingones in time at most $2^{O(n)}$. While clearly a very natural result -- everyone would agree that also with such noise the unimodal \leadingones problem should not become harder than the needle-in-the-haystack problem -- the technical, long, and  problem-specific proof of this result, despite following the intuitive argument just laid out, suggests that such analyses can be harder than one would expect.

In this work, we will argue that such exponential upper bounds are interesting beyond completing a runtime picture of a given problem. We then show that with a different analysis method such uncommon runtime questions can be analyzed relatively easily. As one out of several results, we drastically extend the result in~\cite{Sudholt20} and show that an exponential runtime guarantee holds for 
\begin{itemize}
\item any of the algorithms randomized local search, Metropolis algorithm, simulated annealing, and \oea,
\item when optimizing any weakly monotonic objective function, e.g., \onemax, linear functions, monotone polynomials, \leadingones, plateau functions, and the needle problem,
\item in the presence of all common forms of prior and posterior noise with a noise probability of at most $1-\eps$, $\eps>0$ a constant.
\end{itemize}

\subsection{Exponential Runtime Analysis}

The area of mathematical runtime analysis, established as a recognized subfield of the theory of evolutionary algorithms by Ingo Wegener and his research group, seeks to understand the working principles of evolutionary computation via rigorously proven results on the performance of evolutionary algorithms and other search heuristics in a similar spirit as done in classic algorithms analysis for much longer time. 

Adopting the view of classic algorithmics that runtimes polynomial in the problem size are efficient and larger runtimes are inefficient, the vast majority of the results in this field prove polynomial upper bounds or super-polynomial lower bounds. For two reasons, we feel that also super-polynomial and even exponential runtime guarantees are desirable in the theory of evolutionary algorithms. 

Our first set of arguments is identical to the arguments made in the classic algorithms field, which led to a shift in paradigms and established the field of exact exponential algorithms~\cite{Fomink10,FominK13}. These arguments are that (i)~for many important problems nothing better than exponential time algorithms are known, so one cannot just ignore these problems in algorithms research, (ii)~with the increase of computational power, also exponential time algorithms can be used for problems of moderate (and interesting) size, and (iii)~that the existing research on exponential-time algorithms has produced many algorithms that, while still exponential time, are much faster than na\"ive exponential-time approaches like exhaustive search. 

Our second line of argument is that exponential time algorithms are of additional interest in evolutionary computation for the following reasons. 

(i)~\emph{To increase our understanding of the working principles of evolutionary algorithms.} For a large number of algorithmic problems in our field an exponential lower bound has been proven, but for essentially none of these problems an upper bound better than the trivial $n^{O(n)}$ bound exists. It is clear that matching upper and lower bounds tell us most, not only about the runtimes, but also about the working principles of EAs. Tight bounds naturally have to grasp the true way the EA progresses better. For example, the general $n^{O(n)}$ upper bound for all algorithms using standard bit mutation is based on the simple argument that the optimum can be generated from any search point with probability at least $n^{-n}$. Besides being very pessimistic, this argument does not tell us a lot on how really the EA optimizes the problem at hand (except for the very particular case that the EA is stuck in a local optimum in Hamming distance $n$ to the global optimum). In contrast, as a positive example, the matching $(1 \pm o(1)) e n \ln n$ upper~\cite{Muhlenbein92} and lower~\cite{GarnierKS99,DoerrFW11} bound for the runtime of the \oea on \onemax together with their proofs shows that for this optimization process, the effect of mutations flipping more than one bit has no influence on the runtime apart from lower order terms. In a broader sense, this insight suggests that flipping larger number of bits is mainly useful to leave local optima, but not to make fast progress along easy slopes of the fitness landscape. 

(ii)~\emph{Because understanding runtimes in the exponential and super-exponential regime is important for the application of EAs.} Many classic evolutionary algorithms can easily have a super-exponential runtime. For example, Witt~\cite{Witt05} has shown that the simple \oea has an expected runtime\footnote{As common both in classic algorithms and in our field, by runtime we mean the worst-case runtime taken over all input instances.} of $n^{\Theta(n)}$ on the minimum makespan scheduling problem. Hence knowing that an evolutionary algorithm ``only'' has an exponential runtime can be interesting. 

We note that for problems with exponential-size search spaces (such as the search space $\{0,1\}^n$ regarded exclusively in this work) blind random search and exhaustive search are exponential-time alternatives. For that reason, in addition to knowing that an EA has an exponential runtime guarantee (that is, a runtime of at most $C^{n}$ for some constant $C>1$), it would be very desirable to also have a good estimate for the base of the exponential function, that is, the constant $C$. Unfortunately, at this moment where we just start reducing the trivial $n^{O(n)}$ upper bound to exponential upper bounds, we are not yet in the position to optimize the constants in the exponent. We are optimistic though (and give some indication for this in Section~\ref{sec:constants}) that our methods can be fine-tuned to give interesting values for the base of the exponential function as well. We recall that such an incremental progress is not untypical for the mathematical runtime analysis of EAs -- in the regime of polynomial bounds, subject to intensive research since the 1990s, the leading constants for elementary problems such as \leadingones and linear functions were only determined from 2010 on~\cite{BottcherDN10,Sudholt13,Witt13}.

With this motivation in mind and spurred by the observation that exponential upper bounds are not trivial to obtain, we conduct in this work the first investigation focused on the problem of proving exponential upper bounds for runtimes of EAs. 

\subsection{State of the Art}

We are not aware of many previous works on exponential or super-exponential upper bounds on runtimes of EAs. In the, to the best of our knowledge, first work proving an exponential upper bound for the runtime of an~EA, Droste, Jansen, and Wegener~\cite[Theorem~9]{DrosteJW98ppsn} show that the \oea optimizes the \needle function (called peak function there) in expected time at most $(2\pi)^{-1} n^{1/2} \exp(2n)$. Only a year later, Garnier, Kallel, and Schoenauer~\cite[Proposition~3.1]{GarnierKS99} in a remarkably precise analysis showed that the expected runtime of the \oea on the \needle function is $(1 \pm o(1)) (1 - \frac 1e)^{-1} 2^n$.

A general upper bound of $n^n$ for the expected runtime of the \oea on any pseudo-Boolean function was given in~\cite[Theorem~6]{DrosteJW02} by arguing that each offspring with probability at least $n^{-n}$ is the optimum. This bound is tight as witnessed, among others, by the trap function~\cite[Theorem~8]{DrosteJW02} and the minimum makespan scheduling problem~\cite{Witt05}. Similar arguments as used in~\cite{DrosteJW02} to prove the $n^n$ upper bound showed an upper bound of $4^n \log_2 n$ for the \oea using the mutation rates $2^i / n$, $i = 0, 1, \dots, 2^{\lfloor \log_2 n \rfloor - 1}$, in a cyclic fashion~\cite[Theorem~3]{JansenW06} and an upper bound of $O(n^\beta 2^n)$ for the fast \oea with (constant) power-law exponent $\beta > 1$~\cite[Theorem~5.3]{DoerrLMN17}. With more complex arguments, the self-adjusting SD-(1+1)~EA was shown to optimize the trap function in time $O((2.34)^n \log n)$~\cite[Corollary~2]{RajabiW20}. It is easy to see that when compromising with the base of the exponential function, very similar arguments show an exponential upper bound for the runtime on any function $f : \{0,1\}^n \to \R$.

There are a few analyses for parameterized problems showing bounds that can become exponential or worse when the problem parameter is chosen in an extreme manner. Here the $\Theta(n^k)$ runtime bound for the \oea optimizing jump functions with jump size $k \ge 2$~\cite[Theorem~25]{DrosteJW02} is the best known example. More interesting results have been shown in the context of parameterized complexity~\cite{NeumannS20bookchapter}, but again these results have been derived with small parameter values in mind and thus are most interesting for this case.

In contrast to these sporadic upper bounds, there is a large number of (near-)exponential lower bounds, e.g., for a broad class of non-elitist algorithms with too low selection pressure~\cite{Lehre10}, for some algorithms using fitness-proportionate selection~\cite{HappJKN08,NeumannOW09,Lehre11,Doerr20ppsnLB}, for the simple genetic algorithm with an only moderately large population size~\cite{OlivetoW15}, for the optimization of strictly monotonic functions~\cite{DoerrJSWZ13,Lengler18,LenglerZ19}, and for various problems in noisy optimization~\cite{GiessenK16,QianBJT19,Sudholt20}.

Apart from a single exception, for none of these lower bounds it is known how tight they are, not even a result ruling out a runtime of $n^{\Theta(n)}$. The exceptional exponential upper bound shown in~\cite[Theorem~6]{Sudholt20} reads as follows. Consider optimizing the \leadingones benchmark function defined on bit strings of length $n$ via the \oea. Assume that in each iteration, the fitness evaluation of both parent and offspring is subject to stochastically independent prior noise of one of the following two types. (i)~With probability $p \le \frac 12$, not the true fitness is returned, but the fitness of a random Hamming neighbor. (ii)~With probability $p' \in [0,1]$, the search point to be evaluated is disturbed by flipping each bit independently with some probability $q \le \frac 12$ and the fitness of this disturbed search point is returned, with probability $1 - p'$, the fitness of the original search point is returned; here we assume that $p' \min\{1, qn\} \le \frac 12$. Then the expected optimization time, that is, the number of iterations until the optimum is sampled, is at most exponential in $n$.

With a noise probability of at most $\frac 12$ and a weakly monotonic, that is, weakly preferring $1$-bits over $0$-bits, fitness function one would think that this optimization process in some suitable sense is at least as good as the corresponding process on the \needle function, where absolutely no fitness signal guides the search. This is indeed true, as the proof in~\cite{Sudholt20} shows. Surprisingly, as this proof also shows, it is highly non-trivial to make this intuitive argument mathematically rigorous. The proof in~\cite{Sudholt20} is around four pages long (including the one of the preliminary lemma) and builds on a technical estimate of the mixing time, which heavily exploits characteristics of the \leadingones objective function. Consequently, this proof does not easily generalize to other easy benchmark functions such as \onemax or linear functions. 

\subsection{Our Results}

Observing that the natural approach taken in~\cite{Sudholt20} is unexpectedly difficult, we develop an alternative approach to proving exponential upper bounds. It builds on the following elementary observation. In the, slightly extremal, situation that we aim at an exponential upper bound, we can wait for an exponentially unlikely ``lucky'' way to generate the optimum. Being at most exponentially unlikely, that is, having a probability of $p = 2^{-O(n)}$, it takes $2^{O(n)}$ attempts until we succeed. Hence if each attempt takes at most exponential time~$T_0$ (all our attempts will only take polynomial time), we obtain an exponential upper bound on the expected runtime, and moreover, the distributional bound that the runtime is stochastically dominated by $T_0$ times a geometric distribution with success rate $p$. This general argument (without the elementary rephrasing in the stochastic domination language) was already used in the proof of the $\poly(n) e^{2n}$ upper bound on the expected runtime of the \oea on the \needle function by Droste, Jansen, and Wegener~\cite{DrosteJW98ppsn} more than twenty years ago. It is apparently not very well known in the community, most likely due to the fact that only one year later, Garnier, Kallel, and Schoenauer~\cite{GarnierKS99} presented a much tighter analysis of this runtime via different methods. We are not aware of any other use of this argument, which might explain why it was overlooked in~\cite{Sudholt20} (and we give in that we also learned it only from a review on an earlier version of this work). 

How powerful this simple approach is, naturally, depends on how easy it is to exhibit lucky ways to find the optimum fast. As we shall demonstrate, this is in fact often easy. For example (see Theorem~\ref{thm:main} for the details), it suffices that in each iteration the probability to move to a Hamming neighbor one step closer to the optimum is $\Omega(n^{-1})$. From this, we can show that from any starting point, the probability to reach the optimum in at most $n$ iterations is at least $2^{-O(n)}$. As argued in the preceding paragraph, this yields an expected runtime of $n 2^{O(n)} = 2^{O(n)}$. This argument, without noise and for the \oea only, was also used in the \needle analysis in~\cite{DrosteJW98ppsn}.

Together with some elementary computations, this approach suffices to show that a large number of ${(1+1)}$-type algorithms in the presence of a large variety of types of noise with noise probability at most $1 - \eps$, $\eps > 0$ a constant, optimize any weakly monotonic function (including, e.g., \onemax, \leadingones, and the needle function) in at most exponential time (Theorem~\ref{thm:mono}). 

With similar arguments, we extend this result to the \oea optimizing jump functions with jump size at most $\frac{n}{\ln n}$ in Section~\ref{ssec:jump}. For the particular noise model of bit-wise noise with rate $q$ (search points are disturbed by flipping each bit independently with probability $q$ before the fitness evaluation), we show in Section~\ref{ssec:extreme} that any of the above algorithms even in the presence of extreme noise with $q$ as high as $1-\eps$, $\eps > 0$ a constant, can optimize the \onemax benchmark in exponential time. For the exponential lower bound on the runtime of the \oea with fitness-proportionate selection on linear functions in~\cite{HappJKN08}, we easily derive a matching upper bound in Section~\ref{ssec:fp}. As an example showing that our approach can also yield sub-exponential upper bounds, we prove in Section~\ref{ssec:subexp} that the \oclea with $\lambda \ge (1-\eps) \log_{\frac{e}{e-1}}(n)$, and thus potentially below the threshold for polynomial time, optimizes \onemax in time~$\exp(O(n^{\eps}))$. To ease the presentation, we only regard single-trajectory algorithms in our general result (Theorem~\ref{thm:main}). To show that similar arguments can also be used to analyze algorithms with non-trivial parent populations, we show in Section~\ref{sec:simpleGA} that the mutation-based variant of the simple genetic algorithm finds the optimum of \onemax in exponential time. This matches the known exponential lower bound.

\section{Preliminaries}\label{sec:prelim}

In this section, we describe the algorithms, the noise models, and the benchmark problems considered in this work. Almost all of this is standard, so we aim at brevity and refer to other works for more detail. We only consider optimization problems defined on the search space $\{0,1\}^n$ of bit strings of length $n$; we thus also formulate all algorithms only for this setting. We have no doubt, though, that our methods can also be applied to other discrete optimization problems.

We only use the standard notation of this field. This includes writing $[a..b] := \{z \in \Z \mid a \le z \le b\}$ and denoting by $H(x,y) := |\{i \in [1..n] \mid x_i = y_i\}|$ the \emph{Hamming distance} of two bit strings $x, y \in \{0,1\}^n$. We denote by $\Geom(p)$ the \emph{geometric distribution} with success rate $p \in (0,1]$. Hence if a random variable $X$ is geometrically distributed with parameter $p$, we write $X \sim \Geom(p)$ to denote this, then $\Pr[X = k] = (1-p)^{k-1} p$ for all $k \in \Z_{\ge 1}$. For two random variables $X, Y$ we write $X \preceq Y$ to denote that $Y$ \emph{stochastically dominates} $X$, that is, that $\Pr[X \ge \lambda] \le \Pr[Y \ge \lambda]$ for all $\lambda \in \R$.

\subsection{Algorithms}\label{ssec:algo}

We call a randomized search heuristic \emph{single-trajectory} search algorithm if it is an iterative heuristic which starts with a single solution $x^{(0)}$ and in each iteration $t = 1, 2, \dots$ updates this solution to a solution $x^{(t)}$. We do not make any assumption on how this update is computed. In particular, the next solution may be computed from more than one solution candidate sampled in this iteration. We do, in principle, allow that information other than the search point $x^{(t-1)}$ is taken into iteration $t$. However, in our main technical result we require that the key condition can be checked only from the search point $x^{(t-1)}$. Formally speaking, this means that for any possible history of the search process up to this point, when conditioning on this history, the key condition is true. To ease the language, we shall write ``regardless of what happened in the first $t-1$ iterations'' to express this conditioning. Naturally, for algorithms that can be described via a Markov chain a conditioning on the history is not necessary.

Examples for single-trajectory algorithms are (randomized) local search, the Metropolis algorithm, simulated annealing, and evolutionary algorithms working with a parent population of size one, such as the \oea, \oplea, \oclea, \opllga~\cite{DoerrDE15}, and SSWM algorithm~\cite{PaixaoHST17}. 

We call a single-trajectory algorithm \emph{${(1+1)}$-type algorithm} if in each iteration $t$ it generates one solution $y$ and takes as next parent individual $x^{(t)}$ either $y$ or $x^{(t-1)}$. Among the above examples, (randomized) local search, the Metropolis algorithm, simulated annealing, and the \oea are ${(1+1)}$-type algorithms.

We brief{}ly describe the ${(1+1)}$-type algorithms regarded in this work. Usually all these algorithms start with a solution $x^{(0)}$ chosen uniformly at random from $\{0,1\}^n$. Since none of our results relies on this assumption, we allow any kind of initialization, that is, also without explicit mention all our results hold for any initial search point $x^{(0)}$. We formulate all algorithms for the maximization of a given objective function $f : \{0,1\}^n \to \R$, which as common in evolutionary computation we call \emph{fitness}. 

The \textbf{Randomized local search (RLS)} heuristic in each iteration $t$ generates a new solution $y$ by copying $x^{(t-1)}$ and then flipping a single bit, chosen uniformly at random, in~$y$. If $f(y) \ge f(x^{(t-1)})$, then $x^{(t)} := y$, else $x^{(t)} := x^{(t-1)}$.

The \textbf{Metropolis algorithm} is identical to RLS apart from the selection step. If $f(y) \ge f(x^{(t-1)})$, then again $x^{(t)} := y$. However, if $f(y) < f(x^{(t-1)})$, then still with probability 
\[\exp\left(-\frac{f(x^{(t-1)}) - f(y)}{T}\right)\] 
the algorithm accepts the new solution, that is, $x^{(t)} := y$. Only with probability $1 - \exp(-(f(x^{(t-1)}) - f(y))/T)$ the new solution is discarded, that is, $x^{(t)} := x^{(t-1)}$. The \emph{temperature} $T > 0$ is an algorithm parameter that defines the selection pressure. The Metropolis algorithm with a time-dependent (usually decreasing) temperature $T(t)$ is called \textbf{simulated annealing}.

The \textbf{(1+1) EA} is identical to randomized local search except that now the offspring $y$ is generated by a global mutation operator. In this work, we only consider the most classic choice of mutation, which is standard bit mutation. Here, $y$ is taken as a copy of $x^{(t-1)}$ and then each bit of $y$ is flipped independently with probability $p$, the \emph{mutation rate}. The most common choice, and our choice in this work for all static mutation rates, is $p = \frac 1n$. When the mutation rate is chosen randomly according to a power-law with exponent $\beta > 1$, this algorithm is called \textbf{fast (1+1) EA}~\cite{DoerrLMN17}. More precisely, here a random $\alpha \in [1..n/2]$ is chosen such that $\Pr[\alpha = k] = k^{-\beta} / \sum_{i=1}^{n/2} i^{-\beta}$ and then standard bit mutation is performed with $p = \frac \alpha n$.

From the description of the algorithms, the following property is immediate. In simple words, it says that the algorithms go to any Hamming neighbor that is not worse than the parent with probability $\Omega(\frac 1n)$.

\begin{proposition}\label{prop:algo}
  For any algorithm $A$ described above (and any choice of the parameters not fixed above, that is, temperature $T$, time-dependent temperature $T(t)$, or power-law exponent $\beta > 1$), there is a constant $c_A > 0$ such that the following holds. 
  
  For any iteration~$t$ and any $z$ with $H(z,x^{(t-1)})=1$, and regardless of what happened in the previous iterations, the offspring $y$ generated by $A$ in iteration~$t$ satisfies $\Pr[y = z] \ge \frac{c_A}{n}$. If $f(y) \ge f(x^{(t-1)})$, then also $\Pr[x^{(t)} = z] \ge \frac{c_A}{n}$.  
\end{proposition}

\subsection{Noise Models}\label{ssec:noise}

Optimization in the presence of noise, that is, stochastically disturbed access to the problem instance, is an important topic in the optimization of real-world problems. The most common form are noisy objective functions, that is, that the optimization algorithm does not always learn the correct quality (fitness) of a search point. Randomized search heuristics are generally believed to be reasonably robust to noise, see, e.g.,~\cite{JinB05,BianchiDGG09}, which differs from problem-specific deterministic algorithms, which often cannot cope with any noise. Some theoretical work exists on how randomized search heuristics cope with noise, started by the seminal paper of Droste~\cite{Droste04} and, quite some time later, continued with, among others,~\cite{SudholtT12,DangL15,AkimotoMT15,FriedrichKKS15,GiessenK16,FriedrichKKS17,Sudholt20,QianYZ18,BianQT18,Dang-NhuDDIN18,QianBJT19,DoerrS19}. We refer to the later papers or the survey~\cite{NeumannPR20bookchapter} for a detailed discussion of the state of the art. 

In theoretical studies on how randomized search heuristics cope with noise, the usual assumption is that all fitness evaluations are subject to independently sampled noise. Also, it is usually assumed that whenever the fitness of a search point is used, say in a selection step, then it is evaluated anew. This was not done, e.g., in~\cite{SudholtT12}, but as~\cite{DoerrHK12ants} points out, a much better performance can be obtained when reevaluating solutions. 

Various models of noise have been investigated so far. In \emph{prior noise} models, the search point to be evaluated is subject to a stochastic modification and the algorithm learns the fitness of the disturbed search point (but not the disturbed search point itself). The following prior noise models from the literature will be covered in this work. In this description, we always assume that we try to learn the fitness $f$ of a search point $x$.

\textbf{One-bit noise with probability $p$:} With probability $p$, the fitness of a random Hamming neighbor of $x$ is returned, that is, $y$ is obtained from $x$ by flipping a bit chosen uniformly at random and then $f(y)$ is returned; otherwise, the correct fitness $f(x)$ is returned. An asymmetric version of this noise was considered in~\cite{QianYZ18}. We do not regard this here, but note that all our results hold for this noise as well.

\textbf{Independent bit-flip noise with rate $q$:} From $x$, a search point $y$ is obtained by flipping each bit independently with probability $q$; then $f(y)$ is returned. 

\textbf{$(p,q)$-noise:} With probability $p$, a search point $y$ is obtained from $x$ by flipping each bit independently with probability $q$ and $f(y)$ is returned; otherwise, $f(x)$ is returned. Note that the probability that the fitness of the original search point is returned is $p + (1-p)(1-q)^n$. The $(p,q)$-noise model contains as special case $(1,q)$ the independent bit-flip noise with rate~$q$.

Very roughly speaking, the typical results for these types of noise are that ${(1+1)}$-type algorithms remain efficient when noise occurs with a low probability, e.g., of $O(\frac{\log n}{n})$ for the \onemax benchmark or $O(\frac{\log n}{n^2})$  for the \leadingones benchmark, but that asymptotically higher noise levels lead to super-polynomial runtimes. By using non-trivial populations, resampling, or resorting to other algorithms such as estimation-of-distribution algorithms, the robustness to noise can be considerably improved.

In the \emph{posterior noise} model, the search point $x$ is first correctly evaluated, but then the obtained fitness $f(x)$ is disturbed. The most common posterior noise is \textbf{additive noise}, that is, the returned fitness is $f(x) + X$, where $X$ is a random variable sampled from some given distribution, which does not depend on $x$ (that is, for all search points the difference between the true and the noisy fitness is identically distributed). The most common setting is that $X$ follows a Gaussian distribution. We note that regardless of~$X$, this noise gives a correct comparison of two search points of different quality with probability at least $\frac 12$.
\begin{proposition}\label{prop:better}
  Let $x ,y \in \{0,1\}^n$ with $f(x) > f(y)$ (resp.~$f(x) \ge f(y)$). Let $X, Y$ be independent and identically distributed real-valued random variables. Then $\Pr[f(x) + X > f(y) + Y] \ge \frac 12$ (resp.~$\Pr[f(x) + X \ge f(y) + Y] \ge \frac 12$).
\end{proposition}

\begin{proof}
  By symmetry, we have $\Pr[X \ge Y] = \Pr[Y \ge X]$. Hence 
  \begin{align*}
  1 &= \Pr[(X \ge Y) \vee (Y \ge X)] \\
  &\le \Pr[X \ge Y] + \Pr[Y \ge X] = 2 \Pr[X \ge Y].
  \end{align*}
  This shows $\Pr[X \ge Y] \ge \frac 12$ and thus $\Pr[f(x) + X > f(y) + Y] \ge \Pr[X \ge Y] \ge \frac 12$ (resp.~$\Pr[f(x) + X \ge f(y) + Y] \ge \Pr[X \ge Y] \ge \frac 12$). 
\end{proof} 

Since our aim is showing that also in the presence of extreme noise we still have at most exponential runtimes, we also consider the following \textbf{unrestricted adversarial noise with probability $p$}. In this model, with probability $1-p$ the true fitness is returned. With probability $p$, however, an all-powerful adversary decides the returned fitness value. This adversary knows the algorithm, the optimization problem, and the full history of the optimization process. He does not know, though, the outcome of future random events (both concerning the algorithm and the noise). 

The following basic observation gives an estimate for the probability that a noisy fitness comparison gives the right result. A corresponding result for additive posterior noise, namely that regardless of the noise distribution the fitnesses are compared correctly with probability at least $\frac 12$, was already shown in Proposition~\ref{prop:better}.

\begin{proposition}\label{prop:nonoise}
  Let $\eps > 0$. Let $f : \{0,1\}^n \to \R$. Let $x,y \in \{0,1\}^n$ such that $f(x) \le f(y)$. Consider any noise model described above except the one of additive posterior noise. Assume that $p \le 1-\eps$ in the case of one-bit noise or unrestricted adversarial noise, $(1 - q)^n \ge \eps$ in the case of bit-wise noise, $1 - p(1 - (1 - q)^n) \ge \eps$ in the case of $(p,q)$-noise. Denote by $\tilde f$ the noisy version of $f$ with our convention that each noise evaluation of $f$ uses fresh independent randomness. Then $\Pr[\tilde f(x) \le \tilde f(y)] \ge \eps^2$.
\end{proposition}

\begin{proof}
  Under the conditions named above, with probability at least $\eps$ the noisy fitness returns the true fitness value. Consequently, with probability at least $\eps^2$ this happens for both $x$ and $y$ and we have thus $\tilde f(x) \le \tilde f(y)$.
\end{proof}

\subsection{Benchmark Problems}\label{ssec:problem}

We now describe the benchmark problems used in this paper. They are all well-known and extensively used, so we refer to the literature~\cite{NeumannW10,AugerD11,Jansen13,DoerrN20} for more details. As said earlier, we only regard problems defined on bit-strings of length $n$, hence all functions are $\{0,1\}^n \to \R$. In many respects the easiest benchmark problem is the function \textbf{OneMax} defined by $\OM(x) := \onemax(x) := \|x\|_1 = \sum_{i=1}^n x_i$ for all $x = (x_1, \dots, x_n) \in \{0,1\}^n$. This function is a member of the class of \textbf{linear functions}, which are all functions $f$ such that $f(x) = \sum_{i=1}^n a_i x_i$ for given $a_1, \dots, a_n \in \R$. One often assumes that all $a_i$ are positive, but we do not need this assumption. A function $f$ is called strictly \textbf{monotonic} (or strictly monotonically increasing), where the word strictly often is omitted, when $f(x) < f(y)$ for all $x, y \in \{0,1\}^n$ with $x \le y$ (component-wise) and $x \neq y$. This is equivalent to saying that whenever one flips a zero of the argument into a one, the fitness strictly increases. Clearly, linear functions with positive weights $a_i$ are strictly monotonic. \onemax, arbitrary linear functions, and monotonic functions can all be solved easily via randomized local search (RLS), namely in time $O(n \log n)$, as follows from a classic coupon collector argument. The \oea with mutation rate $\Theta(\frac 1n)$ solves all linear functions in time $O(n \log n)$~\cite{DrosteJW02,DoerrG13algo,Witt13}, but this is less obvious; for the \oea with fast mutation, an $O(n \log n)$ runtime bound is only known for \onemax~\cite{DoerrLMN17}. The situation for strictly monotonic functions is complicated~\cite{Jansen07,DoerrJSWZ13,ColinDF14,Lengler18}, but the latest result~\cite{LenglerMS19} shows at least that the \oea with standard mutation rate $\frac 1n$ solves all strictly monotonic functions in time $O(n \log^2 n)$. For the Metropolis algorithm, a temperature of $T \le \frac{1}{2\ln n}$ is easily seen to suffice to optimize $\onemax$ in time $O(n \log n)$, whereas a temperature of $T \ge \frac{2}{\ln n}$ leads to a negative drift in a large range around the optimum, and thus to a super-polynomial runtime. The precise boundary between polynomial and super-polynomial, unfortunately hard to interpret, was determined in~\cite{KadenWW09}. We are not aware of such a result for simulated annealing, but it is clear that the cooling schedule needs to lead to temperatures for which the Metropolis algorithm is polynomial for a sufficiently long time to ensure a polynomial runtime.  

When noise comes into play, RLS and \oea can stand low noise levels when optimizing \onemax. For the one-bit noise model, any $p = O(\frac{\log n}{n})$ implies a polynomial runtime, higher values of $p$ give super-polynomial runtimes~\cite{Droste04}. Similar results hold for bit-wise noise and $(p,q)$ noise~\cite{GiessenK16,QianBJT19,Dang-NhuDDIN18}. For additive posterior noise, the \oea has a polynomial runtime on \onemax if the variance of the noise distribution is $\sigma^2 = O(\frac{\log n}{n})$. If the noise has a Gaussian distribution, then the runtime is polynomial if the variance  is at most $\sigma^2 \le \frac{1}{4 \ln n}$, and it is super-polynomial, if $\sigma^2 \ge (1+\eps) \frac{1}{4 \ln n}$ for a constant $\eps > 0$~\cite{GiessenK16}.	

Still unimodal (that is, not having true local optima), but not anymore strictly monotonic is the classic \textbf{LeadingOnes} function, which counts the number of ones up to the first zero. Formally, $\leadingones(x) := \max\{i \in [0..n] \mid \forall j \in [1..i]: x_j = 1\}$. Both RLS and \oea optimize \leadingones in time $\Theta(n^2)$. With one-bit noise $p = O(\frac{\log n}{n^2})$ gives a polynomial runtime and larger values of $p$ lead to super-polynomial runtimes of the \oea; again, similar results hold for bit-wise and $(p,q)$ noise~\cite{QianBJT19,Sudholt20}. For additive posterior noise, the runtime of the \oea on \leadingones remains $O(n^2)$ if the noise distribution has a variance of $\sigma \le \frac{1}{12en^2}$. Another unimodal, but not strictly monotonic class of functions that has been the subject of a runtime analysis are \textbf{monotone polynomials} of degree $d$. For these, a (tight) upper bound of $O(2^d \frac nd \log(\frac nd +1))$ has been shown for the expected runtime of RLS~\cite{WegenerW05}. For the \oea, only weaker bounds were proven, but the $O(2^d \frac nd \log(\frac nd +1))$ was conjectured as well~\cite{WegenerW05}.

The classic multimodal benchmark is the class of \textbf{jump functions}. The jump function with \emph{jump parameter (jump size)} $k \in [1..n]$ is defined by
\[
\jump_{nk}(x) = 
\begin{cases}
\|x\|_1+k & \mbox{if $\|x\|_1 \in [0..n-k] \cup \{n\}$,}\\
n - \|x\|_1 & \mbox{if $\|x\|_1 \in [n-k+1\, ..\, n-1]$}.
\end{cases}
\]
Hence for $k = 1$, we have a fitness landscape isomorphic to the one of $\onemax$, but for larger values of $k$ there is a fitness valley (``gap'') $G_{nk} \coloneqq \{x \in \{0,1\}^n \mid n-k < \|x\|_1 < n\}$ 
consisting of the $k-1$ highest sub-optimal fitness levels of the \onemax function. This valley is impossible to cross for RLS, very hard to cross for the Metropolis algorithm (an expected optimization of $\exp(\Omega(n))$ was stated without proof in~\cite{LissovoiOW19}), and hard to cross for the \oea. When using standard bit mutation with mutation rate $\frac 1n$, the probability to generate the optimum from a parent on the local optimum is only $p_k := (1-\frac 1n)^{n-k} n^{-k} < n^k$. For this reason, the runtime of the \oea on $\jump_{nk}$ is $\Theta(n^k)$~\cite{DrosteJW02}. When the fitness in the gap region is not lower than the one of the local optimum, but equal to it (so that there is no local optimum, but all search points in Hamming distance $1$ to $k$ from the optimum have the same fitness $n$), then we have the \textbf{plateau function} $\plateau_{nk}$ introduced in~\cite{AntipovD18}. While not multimodal, it is still difficult to optimize and, for $k$ constant, a runtime of $\Theta(n^k)$ was shown. The special case $\plateau_{nn}$ is the famous \textbf{needle function} or needle-in-the-haystack problem, where all search points apart from the optimum have the same fitness. Since a black-box optimization algorithm optimizing this function receives no useful feedback during the optimization (except when the optimum is found), no better algorithms than exhaustive search can exist for this problem. We are not aware of any proven results on the optimization of jump or plateau functions subject to noise, but we would conjecture that the runtime is roughly the maximum of the noisy runtime on \onemax and the noise-free runtime on the jump or plateau function. In particular for the case of jump functions, this is not obvious since the algorithm could also profit from noise, allowing it to accept a worse search point which is closer to the target. A profit from this effect was mostly ruled out for comma selection~\cite{Doerr20gecco} and we expect a similar situation here.

\subsection{A Technical Tool}

Since we shall use it twice, we present here the additive drift theorem for upper bounds of He and Yao~\cite{HeY01} (see also the recent survey~\cite{Lengler20bookchapter}), which allows to translate an expected additive progress in a random process (or a lower bound on it) into an upper bound on an expected hitting time.   
  
\begin{theorem}\label{tdrift}
  Let $S \subseteq \R_{\ge 0}$ be finite and $0 \in S$. Let $X_0, X_1, \ldots$ be a random process taking values in $S$. Let $\delta > 0$. Let $T = \inf\{t \ge 0 \mid X_t = 0\}$. Assume that for all $t \ge 0$ and all $s \in S \setminus \{0\}$ we have $E[X_t - X_{t+1} \mid X_t = s] \ge \delta$. Then $E[T] \le \frac{E[X_0]}{\delta}$.
\end{theorem}

\section{Proving Exponential Upper Bounds}

We now state our general technical result which in many situations allows one to prove exponential upper bounds without greater difficulties. We formulate our result for single-trajectory algorithms since this is notationally convenient and covers most of our applications. We show an exponential upper bound for an EA with non-trivial parent population in Section~\ref{sec:simpleGA}. The result below is formulated for hitting a general search point $x^*$, but the natural application will be for $x^*$ being the optimum solution. We remind the reader that the key argument of the proof below has already appeared in the conference paper~\cite{DrosteJW98ppsn}, but has, to the best of our knowledge, not been used again since then. 
 
\begin{theorem}\label{thm:main}
  Let $A$ be a single-trajectory search algorithm for the optimization of pseudo-Boolean functions. Let $f : \{0,1\}^n \to \R$ and let $x^* \in \{0,1\}^n$. Assume that we use $A$ to optimize $f$, possible in the presence of noise. Assume that this optimization process satisfies the following property.
  \begin{description}
  \item[(Acc)] There is a number $0 < c \le 1$ such that the following is true. Let $t \ge 1$ and $x, z \in \{0,1\}^n$ such that $x \neq x^*$, $H(x,z) = 1$, and $H(z,x^*) = H(x,x^*) - 1$. Regardless of what happened in the first $t-1$ iterations of optimization process, if $x^{(t-1)} = x$, then $\Pr[x^{(t)} = z] \ge \frac cn$.  
  \end{description}
  Let $T = \min\{t \ge 0 \mid x^{(t)} = x^*\}$. Then $T$ is stochastically dominated by $n \Geom((\frac ce)^{n})$. In particular, $E[T] \le n (\frac ec)^n$. 
\end{theorem}

\begin{proof}
  The key argument of this proof is that any interval of $n$ iterations with probability at least $(\frac ce)^n$ generates $x^*$. More precisely, let $t \ge 0$. We condition on $x^{(t)}$ having any fixed value $x \in \{0,1\}^n$ different from $x^*$. We also condition on the history of the process up to iteration~$t-1$. We now show that conditional on all this, we have $\Pr[x^* \in \{x^{(t+1)}, \dots, x^{(t+n)}\}] \ge (\frac ce)^n$. 
  
  Let $d = H(x,x^*)$. We consider the event $\calE$ that for all $i \in [1..d]$, we have $H(x^{(t+i-1)},x^{(t+i)}) = 1$ and $H(x^{(t+i)},x^*) = d-i$, that is, that the algorithm $A$ reduced the distance to $x^*$ by exactly one in each iteration. Let $\calP$ be the set of all such paths $x = z^{(0)}, \dots, z^{(d)} = x^*$, that is, the set of all $(z^{(0)}, \dots, z^{(d)})$ such that $z^{(0)} = x$ and for all $i \in [1..d]$, we have $H(z^{(t+i-1)},z^{(t+i)}) = 1$ and $H(z^{(t+i)},x^*) = d-i$. Each such path can alternatively be described by the order in which the bits $x$ and $x^*$ differ in are flipped. Consequently, $|\calP| = d!$. By assumption (Acc), the probability that the algorithm follows exactly such a path, that is, that $x^{(t+i)} = z^{(i)}$ for all $i \in [0..d]$, is at least $(\frac cn)^d$. We thus have 
  \begin{align}
  \Pr[\calE] &= \sum_{P \in \calP} \Pr[\text{$A$ follows the path $P$}] \nonumber\\
  & \ge |P| \left(\frac cn \right)^d 
   = \frac{d!}{(\frac nc)^d} \nonumber\\
  & \ge \frac{n!}{(\frac nc)^n} 
   \ge \left(\frac{c}{e}\right)^n,\label{eq:main}
  \end{align}
  where the last line uses $c \le 1$ and the elementary estimate $n! \ge (\frac ne)^n$, see, e.g.,~\cite[(4.13)]{Doerr20bookchapter}.
  
  With this, we now know that any interval of $n$ iterations (``phase'') finds the optimum with probability at least $p = (\frac ce)^n$. Consequently, the number of phases until the optimum is found is stochastically dominated by a geometric law with parameter $p$, see~\cite{Doerr19tcs} for more details on this argument. Since each phase by definition lasts exactly $n$ iterations, the number of iterations until the optimum is found is stochastically dominated by $n \Geom(p)$ and the expected number of iterations is at most $E[T] \le n (\frac ec)^n$. 
\end{proof}

\section{Applications of the Main Result}\label{sec:appli}

Despite its simplicity, Theorem~\ref{thm:main} allows to prove, often without much effort, exponential upper bounds for various different algorithmic settings, as we now show.

\subsection{Noisy Optimization of Weakly Monotonic Functions}\label{ssec:main}

As one such result, we now prove that all ${(1+1)}$-type algorithms discussed in Section~\ref{ssec:algo} optimize any weakly monotonic function in at most exponential time even in the presence of any noise discussed in Section~\ref{ssec:noise} as long as the noise probability is at most $1-\eps$, $\eps > 0$ a constant, in the cases of prior or adversarial noise. We recall that the only previous result in this direction~\cite{Sudholt20} shows this claim in the particular case of the \oea optimizing the \leadingones function subject to one-bit or $(p,q)$ prior noise with noise probability at most $\frac 12$.

We say that a function $f : \{0,1\}^n \to \R$ is \emph{weakly monotonic} (or weakly monotonically increasing) if for all $x, y \in \{0,1\}^n$ the condition $x \le y$ (component-wise) implies $f(x) \le f(y)$. The class of weakly monotonic functions includes, obviously, all strictly monotonic functions and thus in particular the classic benchmarks \onemax and linear functions (with non-negative coefficients). However, this class also contains more difficult functions like \leadingones, monotone polynomials, plateau functions, and the needle function. 

\begin{theorem}\label{thm:mono}
  Let $\eps > 0$ be a constant. Let $A$ be one of the randomized search heuristics RLS, the Metropolis algorithm, simulated annealing, or the \oea using standard bit mutation with mutation rate $\frac 1n$ or using the fast mutation operator with $\beta > 1$. Let $f : \{0,1\}^n \to \R$ be any weakly monotonic function. Assume that $A$ optimizes $f$ under one of the following noise assumptions: one-bit noise or unrestricted adversarial noise with $p \le 1-\eps$, bit-wise noise with $(1-q)^n \ge \eps$, $(p,q)$-noise with $1 - p(1 - (1-q)^n) \ge \eps$, or posterior noise with an arbitrary noise distribution.
  
  Then there is a constant $C > 1$, depending only on $\eps$ and the choice of $A$,  such that the time $T$ to sample the optimum $(1, \dots, 1)$ of $f$ is stochastically dominated by $n \Geom(C^{-n})$. In particular, the expected optimization time is at most $E[T] \le nC^{n}$.
\end{theorem}

\begin{proof}
  By Theorem~\ref{thm:main}, it suffices to show that condition~(Acc) is satisfied for $x^* = (1, \dots, 1)$. To this aim, let $x, z \in \{0,1\}^n$ such that $H(x,z) = 1$ and $H(x,x^*) = H(z,x^*)+1$. Assume that for some iteration~$t$ the parent individual satisfies $x^{(t-1)} = x$. By Proposition~\ref{prop:algo}, there is a constant $c_A$ such that the offspring $y$ generated by $A$ in this iteration is equal to $z$ with probability at least $\frac{c_A}{n}$. By the weak monotonicity of $f$, we have $f(z) \ge f(x)$. By Proposition~\ref{prop:better} or~\ref{prop:nonoise}, there is a constant $c_N = \min\{\frac 12, \eps^2\}$ depending on the noise model such that the noisy evaluations of both $x^{(t-1)}$ and $y=z$ in iteration~$t$ with probability at least $c_N$ return an at least as good fitness value for $z$ as for $x$. In this case, $A$ accepts $z$ with probability one, that is, we have $x^{(t)} = z$. In summary, we have shown $\Pr[x^{(t)} = z] \ge \frac{c_A c_N}{n}$ as desired. Now Theorem~\ref{thm:main} immediately gives the claim of Theorem~\ref{thm:mono} with $C = \frac{e}{c_A c_N}$.
\end{proof}

\subsubsection{Discussion: The Base of the Exponential}\label{sec:constants}

In this first work on exponential upper bounds, we did not try to optimize the base of the exponential function, that is, the constant $C$ such that the upper bound is $\poly(n) C^n$. As discussed in the introduction, this constant is important to understand the performance of the algorithm and to compare it to other exponential time algorithms such as blind random search and exhaustive search. 

Since optimizing implicit constants often is highly non-trivial and technical, we shall not start this endeavor here, but only brief{}ly describe how our current approach compares to the one of~\cite{Sudholt20} and where we see room for improving our constants, possibly below the $\Theta(2^n)$ runtime of blind random search and exhaustive search. 

We first note that the proof of~\cite{Sudholt20}, which also is not optimized for giving good constants, shows an upper bound that is at least $\exp(3en) \ge (3480)^n$. Under the noise assumptions taken in~\cite{Sudholt20}, we have a probability of at least $c_N \ge \frac 14$ that parent and offspring are not subject to noise. Regarding the \oea, the probability that a particular Hamming neighbor of the parent is generated as offspring is at least $\frac{1}{en}$, that is, in the notation of the proof of Theorem~\ref{thm:main} we have $c_A = \frac 1e$. This would give a runtime bound of at most $n (\frac{e}{c_A c_N})^n = n (4e^2)^n \le n (30)^n$. For other scenarios, the constant is slightly better. For example, for RLS, Metropolis algorithm, and simulated annealing, we have $c_A = 1$. For posterior noise, $c_N$ can be taken as $\frac 12$ as shown in Proposition~\ref{prop:better}. Hence for these combinations, the above runtime could be estimated by $n (2e)^n \le n (5.5)^n$.

We now argue why we are optimistic that with additional arguments, the constants can be improved when taking into account the particular situations, that is, the noise model, the objective function, and the EA. For example, when optimizing any weakly monotonic function subject to 1-bit noise, we accept an offspring strictly dominating the parent (in the proof of Theorem~\ref{thm:main}) except when the noise flips a zero-bit of the parent or a one-bit of the offspring. This undesired event happens with probability at most $\frac 12$, hence we have $c_N = \frac 12$ instead of $c_N = \frac 14$ before. 

When taking more details of the fitness function into account, we have additional arguments to reduce the impact of the noise. When optimizing \onemax, for example, 1-bit noise is detrimental (in our proof) only if both a one-bit of the offspring and a zero-bit of the parent is flipped, increasing $c_N$ further to $c_N = (1 - O(\frac 1n)) \frac {15}{16}$. 

Taking into account the particular algorithm also allows to fine-tune the analysis. Let us take the \oea as an example. If in our main argument we do not wait for the lucky event that in each iteration we approach the target by one Hamming step, but by two steps, then the number of different ways to go from a starting point in (pessimistic) Hamming distance $n$ to the optimum reduces by a factor of $2^{n/2}$, but we need to pay the price of $\frac 1e$ for flipping exactly one bit only $n/2$ times instead of $n$ times. So we reduce the runtime bound by a factor of $(2/e)^{n/2}$. Note that in addition now one-bit noise has no chance to prevent us from accepting the desired parent if we optimize \onemax.

Finally, a mild understanding of the optimization process can be exploited. In our analysis, we pessimistically pretended that we start a run towards the target from a search point with maximal distance. Unless we are optimizing a very deceptive problem, one should be able to argue that the typical starting point will be rather in distance $\frac n2$, more precisely, that is takes an expected short time to reach such a search point, and then start the run towards the target from there. Note that such an argument does not need a mixing time argument as precise as in~\cite{Sudholt20}. With a little more understanding, we might also be able to argue that the algorithm easily (that is, in expected polynomial time) reaches a search point that is even closer to the optimum. Depending on the strength of the noise and the fitness function, it can be possible to argue that up to a certain fitness level, the influence of the fitness outweighs the negative effect of the noise. Consequently, the resulting drift lets the algorithm quickly approach this fitness level. Such arguments have been used, e.g., in~\cite{Dang-NhuDDIN18}.

We shall not elaborate on these ideas further, but in summary they give us some optimism that further, more detailed and problem-specific studies will be able to show exponential upper bounds which are of order $C^n$ for a constant $C$ that is less than $2$ (which would prove the algorithm superior to random search).

\subsection{Noisy Optimization of Jump Functions}\label{ssec:jump}

To show the versatility of our general approach, we continue with a number of results of varying flavor. We first show that the \oea can optimize noisy jump functions with jump size at most $\frac{n}{\ln n}$ in exponential time.

\begin{theorem}
  The result of Theorem~\ref{thm:mono} holds also for the \oea optimizing $\jump_{nk}$ when $k \le \frac{n}{\ln n}$.
\end{theorem}

We omit a formal proof since it is very similar to the proof of Theorem~\ref{thm:main} and~\ref{thm:mono}. Let $x$ be any search point outside the gap region of the jump function and let $x = z^{(0)}, z^{(1)}, \dots, z^{(d)} = (1,\dots,1)$ be any shortest path in the Hamming cube from $x$ to the optimum. Note that $z^{(d-k)}$ is a local optimum of the jump function $\jump_{nk}$. We say that the \oea follows this path from iteration~$t$ on if $x^{(t+i)} = z^{(i)}$ for all $i \in [0..d-k]$ and $x^{(t+d-k+1)} = (1, \dots, 1)$. Since the assumptions of this theorem imply condition (Acc) of Theorem~\ref{thm:main} (see the proof of Theorem~\ref{thm:mono}), the probability that the \oea follows this path is at least $(\frac cn)^{d-k}  n^{-k} (1 - \frac 1n)^{n-k}$. When applying a union bound over all these paths, we have to take care that $k!$ paths describe the same event. 

Consequently, the probability that the \oea follows some path from $x$ to $(1,\dots,1)$ is at least $\frac{d!}{k!} (\frac cn)^d (1 - \frac 1n)^{n-k} \ge \frac{1}{k!} (\frac ce)^n \frac 1e$ as in the proof of Theorem~\ref{thm:main}. Since $k \le \frac{n}{\ln n}$, we have $k! \le k^k \le (\frac{n}{\ln n})^{\frac{n}{\ln n}} = \exp(\frac{n}{\ln n} \ln(\frac{n}{\ln n})) \le \exp(n)$. Hence the probability that the optimum is found in $d-k+1$ iterations is at least $(\frac c {e^2})^n \frac 1e$. The remainder of the proof is analogous to the one of Theorem~\ref{thm:main}.

\subsection{Optimization of OneMax Under Extreme Bit-Wise Noise}\label{ssec:extreme}

The following result shows that our general method can also exploit particular noise models. Here, for example, we show that \onemax can be optimized in exponential time even in the presence of bit-wise noise with constant rate $q<1$. Recall that this means that the search point to be evaluated is disturbed in an expected number of $qn$ bits! 

\begin{theorem}
  Let $\eps>0$ be a constant. Let $A$ be one of the randomized search heuristics RLS, the Metropolis algorithm, simulated annealing, or the \oea using standard bit mutation with mutation rate $\frac 1n$ or using the fast mutation operator with $\beta > 1$. Consider optimizing the \onemax benchmark function via $A$ in the presence of bit-wise noise with rate $q \le 1-\eps$. Then the expected time to find the optimum is at most $n C^n$, where $C$ is a constant depending on $\eps$ and the algorithm used.
\end{theorem}

The theorem above contains as special case $q=\frac 12$ the situation that the noisy fitness returns the fitness of a search point uniformly distributed in $\{0,1\}^n$. Here, obviously, the algorithm cannot gain any advantage from the fitness evaluations. 

\begin{proof}
  The proof is identical to the one of Theorem~\ref{thm:mono} except that we cannot invoke Proposition~\ref{prop:nonoise} to argue that with constant probability the better of two Hamming neighbors also has the higher noisy fitness. So we complete the proof by showing this missing statement. To this aim, let $x, y \in \{0,1\}^n$ such that $H(x,y) = 1$ and $\onemax(x) \le \onemax(y)$, which implies $\onemax(y) = \onemax(x)+1$. By assumption, there is a unique bit position $i \in [1..n]$ such that $x_i \neq y_i$ (``interesting bit''), and for this we have $x_i = 0$ and $y_i=1$.
  
  Let $\tilde f_x, \tilde f_y$ be one-time samples of the noisy fitness values of $x$ and $y$. We have $\tilde f_x = f(x \oplus m^x)$, where $\oplus$ denotes addition in $\Z_2$ (or logical XOR) and $m^x \in \{0,1\}^n$ is random such that each bit of $m^x$ is $1$ with probability $q$. Likewise, let $m^y$ be such that $\tilde f_y = f(y \oplus m^y)$.  
  
  With probability $(1-q)^2$, we have $m^x_i = m^y_i = 0$, that is, the interesting bit is not affected by the noise. Conditional on this, we have $\tilde f_x = \|(x \oplus m^x)_{|[1..n] \setminus \{i\}}\|_1 + 0$ and $\tilde f_y = \|(y \oplus m^y)_{|[1..n] \setminus \{i\}}\|_1 + 1$. The two random variables $\|(x \oplus m^x)_{|[1..n] \setminus \{i\}}\|_1$ and $\|(y \oplus m^y)_{|[1..n] \setminus \{i\}}\|_1$ are identically distributed. By symmetry, we have $\|(x \oplus m^x)_{|[1..n] \setminus \{i\}}\|_1 \le \|(y \oplus m^y)_{|[1..n] \setminus \{i\}}\|_1$ with probability at least $\frac 12$. Consequently, $\Pr[\tilde f_x < \tilde f_y] \ge \frac 12 (1-q)^2 = \frac 12 \eps^2$. This shows the missing claim and completes the proof.
\end{proof}

\subsection{Fitness Proportionate Selection}\label{ssec:fp}

Our general method is not restricted to the analysis of noisy optimization. In this subsection, we prove an upper bound matching an exponential lower bound proven in~\cite{HappJKN08}. The main result of~\cite{HappJKN08} is that the \oea needs at least exponential time to optimize any linear function with positive coefficients when the usual elitist selection is replaced by fitness-proportionate selection. Here an offspring $y$ of the parent $x$ is accepted with probability $\frac{f(y)}{f(x)+f(y)}$. We now show that this result is tight, that is, that an exponential number of iterations suffices to optimize any linear function with this algorithm. This result is true for all weakly monotonic functions.

\begin{theorem}
  Let $A$ be the \oea with fitness-proportionate selection. Let $f: \{0,1\}^n \to \R_{>0}$ be any weakly monotonic function. Then the first iteration~$T$ in which the optimum of $f$ is generated satisfies $E[T] \le (2e^2)^n$.
\end{theorem}

\begin{proof}
  We use our main tool Theorem~\ref{thm:main}. Denote by $x^* = (1, \dots, 1)$ this particular optimum of $f$. The probability that the \oea generates a particular Hamming neighbor of the parent as offspring is $\frac 1n (1-\frac 1n)^{n-1} \ge \frac{1}{en}$. The probability that a Hamming neighbor with better fitness accepted is at least~$\frac 12$. This shows condition (Acc) of Theorem~\ref{thm:main} with $c = \frac 1 {2e}$. By the theorem, the expected runtime is at most $(2e^2)^n$. 
\end{proof}

\subsection{Subexponential Upper Bounds}\label{ssec:subexp}

We now show that our method is not restricted to showing runtime bounds that are exponential in the problem dimension. We recall that the \oclea is a simple non-elitist algorithm working with a parent population of size one, initialized with a random individiual. In each iteration, the algorithm creates independently $\lambda$ offspring via standard bit mutation (here: with mutation rate $\frac 1n$) and takes a random best offspring as new parent. In their very precise determination of the efficiency threshold of the \oclea on \onemax, Rowe and Sudholt~\cite{RoweS14} showed that the \oclea has a runtime of at least $\exp(\Omega(n^{\eps/2}))$ when $\lambda \le (1-\eps) \log_{\frac{e}{e-1}}(n)$, $\eps > 0$ a constant. We now show an upper bound of $\exp(O(n^\eps))$ for this runtime. We do not know what is the right asymptotic order of the exponent. From the fact that there is a considerable negative drift when the fitness distance is below $\frac {n^\eps}{2\lambda}$, we would rather suspect that also a lower bound of $\exp(\Omega(\frac{n^\eps}{\lambda}))$ iterations, and hence $\lambda \exp(\Omega(\frac{n^\eps}{\lambda}))$ fitness evaluations, comes true. Since this is not the main topic of this work, we leave this an open problem. 

\begin{theorem}
  Let $0 < \eps < 1$ be a constant. Then there is a constant $C_\eps$ such that for all $\lambda \ge (1-\eps) \log_{\frac{e}{e-1}}(n)$ the expected runtime of the \oclea on \onemax is at most $\exp(C_\eps n^\eps)$.
\end{theorem}

\begin{proof}
  Our proof uses some of the arguments given in~\cite{RoweS14}, however, since we do not require a precise analysis of the easy part of the process, we simply resort to the additive drift theorem (Theorem~\ref{tdrift}) for that part. For any search point $x \in \{0,1\}^n$, denote by $d(x) := n - \onemax(x)$ its fitness distance (which also is its Hamming distance) to the optimum.
  
  For the easy part that the fitness distance is more than $d_0 := \frac{2e^2 n^\eps}{\lambda}$, we first show that starting from an arbitrary solution, with probability at least $\frac 12$ within $(1+o(1)) \frac 2e n^{2-\eps}$ iterations, a solution $x$ with $d(x) \le d_0$ is found. To ease the notation, we show our claim for the initial solution $x^{(0)}$. Since the sequence of search points $x^{(t)}$ generated by the \oclea is a Markov chain, the claim holds analogously for any starting iteration. 
  
  Hence let $x^{(0)}$ be arbitrary. For $t = 0, 1, 2, \dots$, let $X_t = \max\{0,d(x^{(t)}) - d_0\}$ the distance to our target of having a $d$-value of at most $d_0$. Let $t$ be such that $X_{t-1} > 0$. By Lemma~7 of~\cite{RoweS14}, we have $E[X_{t} - X_{t-1} \mid X_t > X_{t-1}] \le e$. The event $X_t > X_{t-1}$ itself means that all $\lambda$ offspring generated in iteration~$t$ have a fitness inferior to the one of the parent. For an offspring to have an inferior \onemax fitness, it is necessary that at least one of the $1$-bits of the parent was flipped. This happens with probability at most $1 - (1 - \frac 1n)^{n-d_0} \le 1-\frac 1e$, since $X_{t-1}>0$. Consequently, 
  \[\Pr[X_t > X_{t-1}] \le (1 - \tfrac 1e)^\lambda \le n^{-(1-\eps)}.\] 
  In summary, 
  \begin{align*}
  E[&{}\!\max\{0, X_{t} - X_{t-1}\}] \\
  &=  E[X_{t} - X_{t-1} \mid X_t > X_{t-1}] \cdot \Pr[X_t > X_{t-1}] \\
  &\le e n^{-(1-\eps)}.
  \end{align*}
  
  For the drift towards the target, we first observe that an offspring $y$ is better than the parent $x$ with probability at least 
  \[d(x) \tfrac 1n (1 - \tfrac 1n)^{n-1} \ge \frac{d(x)}{en}.\] 
  Consequently, the probability that at least one offspring is better, is at least $1 - (1 - \frac{d(x)}{en})^\lambda$. By the elementary, but very useful Lemma~8 of~\cite{RoweS14}, this is at least 
  \begin{align*}
  1 - \frac{1}{\frac{\lambda d(x)}{en} + 1} & \ge 1 - \frac{1}{\frac{\lambda d_0}{en} + 1} \\
  &= 1 - \frac{1}{2e n^{-(1-\eps)}+1} \\
  &= \frac{2e n^{-(1-\eps)}}{2e n^{-(1-\eps)}+1} \ge (1-o(1)) 2e n^{-(1-\eps)}.
  \end{align*}
  Recalling that a fitness improvement means that the process $(X_t)$ goes down by at least one, we have, $E[\max\{0, X_{t-1} - X_{t}\}] =  E[X_{t-1} - X_{t} \mid X_t > X_{t-1}] \cdot \Pr[X_t > X_{t-1}] \ge 1 \cdot (1-o(1)) 2e n^{-(1-\eps)}$.
  
  We have just computed that the drift of the process $(X_t)$ satisfies 
  \begin{align*}
  E[&X_{t-1} - X_t \mid X_{t-1} > 0] \\
  &\ge E[\max\{0, X_{t-1} - X_{t}\}] - E[\max\{0, X_{t} - X_{t-1}\}] \\
  &\ge (1-o(1)) 2e n^{-(1-\eps)} - e n^{-(1-\eps)} = (1-o(1))e n^{-(1-\eps)}.
  \end{align*}
   Consequently, by the additive drift theorem (Theorem~\ref{tdrift}), the first time $T$ to have $d(x^{(T)}) \le d_0$ satisfies 
  \[E[T] \le \frac{n}{(1-o(1))e n^{-(1-\eps)}} = (1+o(1)) \tfrac 1e n^{2-\eps}.\] 
  By Markov's inequality, with probability at least $\frac 12$, $2E[T] \le (1+o(1)) \frac 2e n^{2-\eps}$ iterations suffice to have a parent individual $x^{(t)}$ with $d(x^{(t)}) \le d_0$. 
  
  We now argue, reusing some arguments from the proof of Theorem~\ref{thm:main}, that now another $d_0$ iterations suffice to find the optimum with probability at least $\exp(-C_e n^\eps)$. In a similar manner as above, we see that if the current parent individual has a fitness distance of $d \le d_0$, then the probability to increase the fitness of the parent in one iteration $1 - (1 - \frac{d}{en})^\lambda \ge \frac{\lambda d}{en + \lambda d} \ge \frac{\lambda d}{(e+2e^2)n}$ by our assumption that $d \le d_0$ and the blunt estimate $\lambda d \le 2e^2 n$.

  Consequently, the probability that $d_0$ iterations suffice to find the optimum is at least 
  \begin{align*}
  \prod_{d=1}^{d_0} \frac{d \lambda}{(e+2e^2)n} 
  &= \left(\frac{\lambda}{(e+2e^2)n}\right)^{d_0} (d_0)! \ge \left(\frac{d_0 \lambda}{e(e+2e^2)n}\right)^{d_0} \\
  & = \exp\left(-d_0 \ln\left(\frac{e(e+2e^2)n}{d_0 \lambda}\right)\right)\\
  & = \exp\left(-\frac{2e^2 n^\eps}{\lambda} \ln\left(\frac{e(e+2e^2)n^{1-\eps}}{2 e^2}\right)\right) \\
  &\le \exp(-C'_\eps n^\eps)
  \end{align*}
  for a suitable constant $C'_\eps$; note that we used $\lambda = \Omega(\log n)$ in the last step.
  
  Now as in the proof of Theorem~\ref{thm:main}, the runtime is stochastically dominated by $(1+o(1)) \frac 2e n^{2-\eps} + d_0$ times a geometric random variable with success rate $p = \frac 12 \exp(-C'_\eps n^\eps)$, which gives the claimed expected runtime of at most $(1+o(1)) (\frac 2e n^{2-\eps} + d_0) \frac 1p = \exp(C_\eps n^\eps)$ for a suitable choice of $C_\eps$. 
\end{proof}

\section{An Exponential Upper Bound for the Mutation-based Simple Genetic Algorithm}\label{sec:simpleGA}

To indicate that our general analysis method is not restricted to ${(1+1)}$-type algorithms, we now prove an exponential upper bound for the mutation-based simple genetic algorithm (simple GA), which uses a non-trivial parent population and which is a generational GA, that is, individuals from the current population are never taken into the next population. We believe that, in principle, this algorithm could be covered via an extension of the framework regarded in Theorem~\ref{thm:main}, but for the sake of readability we prefer to give an independent analysis even if this requires repeating (in a compact manner) some arguments from the proof of Theorem~\ref{thm:main}.

The simple GA traditionally is used with crossover~\cite{Goldberg89}. The mutation-only version has been regarded in the runtime analysis community mostly because runtime analyses for crossover-based algorithms are extremely difficult. For example, the first runtime analysis of the crossover-based version is only from 2012~\cite{OlivetoW12gecco} and only shows results for relatively small population sizes below $n^{1/8}$ (the current best result~\cite{OlivetoW15} covers population sizes up to slightly below $n^{1/4}$), whereas for the mutation-based version  already in~2009 a strong result valid for all polynomial population size was shown~\cite{NeumannOW09}.

We start by making this algorithm precise and then discuss the known results. The mutation-only version of the simple GA with population size $\mu \in \N$ for the optimization of a non-negative fitness function $f : \{0,1\}^n \to \R_{\ge 0}$ is described in Algorithm~\ref{alg:simpleGA}. The algorithm starts with a population $P^{(0)}$ of $\mu$ random individuals from $\{0,1\}^n$ (our result will be valid for any initial population, but the usual initialization is random). We view a population $P$ as an array (tuple) $P = (P_1, \dots, P_\mu)$ of not necessarily distinct individuals. In each iteration $t = 1, 2, 3, \dots$, the simple GA computes from the previous population $P^{(t-1)}$ a new population $P^{(t)}$ by $\mu$ times independently selecting an individual from $P^{(t-1)}$ via fitness proportionate selection and mutating it via standard bit mutation with mutation rate $p = \frac 1n$. Here fitness proportionate selection from $P^{(t-1)}$ means that the $i$-th individual $P^{(t-1)}_i$ is chosen with probability $f(P^{(t-1)}_i) / \sum_{j=1}^\mu f(P^{(j)})$; except when all individuals have fitness zero, then fitness proportionate selection agrees with uniform selection, that is, the $i$-th individual is chosen with probability $\frac 1 \mu$.

\begin{algorithm2e}%
	Initialize $P^{(0)}$ with $\mu$ individuals chosen independently and uniformly at random from $\{0,1\}^n$\;
	\For{$t = 1, 2, \ldots$}{
    \For{$i \in [1..\mu]$}{
      Select $x \in P^{(t-1)}$ via fitness proportionate selection\;
      Generate $P^{(t)}_i$ from $x$ by flipping each bit independently with probability $p = \frac 1n$\;
      }
  }
\caption{The simple genetic algorithm (simple GA) with population size $\mu$ and mutation rate $p = \frac 1n$ to maximize a function $f : \{0,1\}^n \to \R_{\ge 0}$.}
\label{alg:simpleGA}
\end{algorithm2e}

The existing runtime results for the simple GA are not very encouraging for this algorithm. Apart from results for very small mutation rates $p = O(n^{-2})$, see~\cite[Theorem~16]{DangL16algo} and~\cite[Theorem~22]{DoerrK19arxiv}, or results for exponentially transformed fitness functions~\cite[Theorem~13 and 14]{NeumannOW09}, they show that the selection pressure resulting from fitness proportionate selection is not strong enough to give a sufficient progress to the optimum. The precise known results for the performance of Algorithm~\ref{alg:simpleGA} on the \onemax benchmark are the following. \cite[Theorem~8]{NeumannOW09} showed that the simple GA with $\mu \le \poly(n)$ needs with high probability more than $2^{n^{1-O(1/\log\log n)}}$ iterations to find the optimum of the \onemax function or any search point in Hamming distance at most $0.003n$ from it. This is only a sub\-exponential lower bound. In~\cite[Corollary~13]{Lehre11}, building on the lower bound method from~\cite{Lehre10}, a truly exponential lower bound is shown and this for weaker the task of finding a search point in Hamming distance at most $0.029n$ from the optimum, but only for a relatively large population size of $\mu \ge n^3$ (and again $\mu \le \poly(n)$). The conditions $\mu \ge n^3$ and $\mu \le \poly(n)$ were removed in~\cite{Doerr20ppsnLB}. Lehre~\cite{Lehre11} also showed an exponential lower bound for the time to reach the optimum for a (mildly) scaled version of fitness proportionate selection and for general $\Theta(1/n)$ mutation rates. For the crossover-based version of the simple GA the current best results is that with population size $\mu \le n^{(1/4) - \eps}$ it takes at least time $2^{n^{\eps/11}}$ to find a solution that is a small constant factor better than $\frac n2$, which is the expected fitness of a random solution.

For the mutation-based version of the simple GA (Algorithm~\ref{alg:simpleGA}) we now show that the known exponential lower bounds are tight, that is, regardless of the population size (as long as it is not super-exponential) the simple GA finds the optimum of the \onemax benchmark in expected time $\exp(O(n))$.

\begin{theorem}\label{thm:simpleGA}
  Consider a run of the simple GA with arbitrary population size $\mu$ and with an arbitrary initial population $P^{(0)}$ on the \onemax benchmark. Let $T$ be the first iteration in which the optimum of \onemax is generated. Then $T \preceq O(n) \Geom(-\exp(O(n))$ and $E[T] = \exp(O(n))$. In particular, if $\mu = \exp(O(n))$, then the optimum is sampled within an exponential number of fitness evaluations. 
\end{theorem}

\begin{proof}
  We follow the rough outline of the proof of Theorem~\ref{thm:main}. As there, it suffices to show that for any state of the algorithm the probability that the optimum is found in the next $O(n)$ iterations is at least $\exp(- O(n))$. Then a simple restart argument as in the proof of Theorem~\ref{thm:main} shows that $T \preceq O(n) \Geom(-\exp(O(n))$ and thus $E[T] = \exp(O(n))$.
  
  \textbf{Obtaining a fitness of at least $\frac n3$:}
  To have a better control over the individuals selected via fitness proportionate selection, it helps to have at least one individual with fitness at least $\frac n3$ in the population. To this aim, we now show that when starting the simple GA with an arbitrary population, it takes an expected time of at most $3n$ iterations to have at least one individual in the population that has fitness at least $\frac n3$.
  
  Let $P^{(0)}$ be an arbitrary initial population and let $P^{(t)}$, $t = 0, 1, 2, \dots$ be the populations generated in a run of the simple GA started with $P^{(0)}$. We regard the following random process $X_t$, $t = 0, 1, 2, \dots$. If $P^{(t)}$ contains at least one individual with fitness $\frac n3$ or more, or if $t \ge 1$ and $X_{t-1} = n$, then let $X_t := n$. Otherwise, let $X_t$ be the average fitness of the individuals in $P^{(t)}$. This defines a random process in a finite subset of $[0,n]$. We estimate $T_0 := \min\{t \mid X_t = n\}$. Assume that for some $t$ we have $X_t < n$. We compute the expected fitness of an individual in $P^{(t+1)}$. Such an individual $y$ is generated from a parent $x$ chosen via fitness proportionate selection from $P^{(t)}$ by flipping each bit independently with probability $p = \frac 1n$. We compute $E[\OM(y) \mid x] = \OM(x) + (n-\OM(x)) \frac 1n - \OM(x) \frac 1n \ge \OM(x) + \frac 13$, using that the $\OM(x) \le \frac n3$. Since fitness-proportionate selection favors better individuals, $\OM(x)$ stochastically dominates the fitness of a random individual $X$ from $P^{(t)}$, see~\cite[Lemma~12]{Doerr20ppsnLBarxiv}. Consequently, $E[\OM(y) - \OM(X)] \ge E[\OM(y) - \OM(x)] \ge \frac 13$. Note that $E[\OM(y)]$ is also the expected average fitness of $P^{(t+1)}$. Since $X_{t+1}$ is this average fitness (which is always at most $n$) or $n$ (in the case that $P^{(t+1)}$ contains an individual with fitness at least $\frac n3$), we have $E[X_{t+1} - X_t] \ge E[\OM(y) - \OM(x)] \ge \frac 13$. Our computation just made was under the assumption that $X_t < n$. Hence we have $E[X_{t+1} - X_t \mid t < T_0] \ge \frac 13$. By the additive drift theorem (Theorem~\ref{tdrift}), we have $E[T_0] \le 3n$. Via a simple Markov bound, we see that with probability at least $\frac 12$, there is a $t \in [0..6n]$ such that $P^{(t)}$ contains an individual with fitness at least $\frac n3$. 
  
  \textbf{From an individual with fitness at least $\frac n3$ to the optimal solution:}
  For any population of the simple GA, an individual $x$ with fitness $\OM(x) \ge \frac n3$ in each selection of a particular parent has a probability of  
  \[q = \frac{\OM(x)}{\sum_{y \in P} f(x)} \ge \frac{n/3}{\mu n} = \frac{1}{3\mu}\]
  of being selected. Hence with probability $1 - (1-q)^\mu \ge 1 - (1-\frac{1}{3\mu})^\mu \ge 1 - \exp(-\frac 13)$, it is selected at least once as parent (here we used the basic estimate $1+r \le e^r$ valid for all $r \in \R$). For the first offspring generated from $x$ in this iteration (and for any other, but we want to regard a particular one), the probability that a particular Hamming neighbor of $x$ results from mutating $x$ is at least $\frac 1n (1-\frac 1n)^{n-1} \ge \frac{1}{en}$. Via an elementary induction we see that if at some time $t$, say for simplicity $t = 0$, we have an individual $x^{(0)}$ in the population with $d := n - \OM(x^{(0)}) \le \frac 23 n$, then the probability that at each time $t \in [1..d]$ there is an $x^{(t)} \in P^{(t)}$ such that $\OM(x^{(t)}) = n - (d-t)$ and $x^{(t)}$ has been generated from $x^{(t-1)}$ by flipping a single bit, is at least $(\frac 1e \frac{1}{e} (1 - \exp(-\frac 13)))^n$, where this number is computed as in~\eqref{eq:main}.
  
  This shows that for any time $t$ of the run of the simple GA, regardless of the past, the probability that the optimum of \onemax is sampled in iterations $t, \dots, t+7n-1$ is at least $\exp(O(n))$. This completes the proof of Theorem~\ref{thm:simpleGA}. 
\end{proof}

\section{Conclusion and Outlook}

In this work, we argued for proving exponential runtime guarantees for evolutionary algorithms. Exponential-time algorithms have been a modern subfield of classic algorithms for around twenty years now for various reasons. Exponential upper bounds for evolutionary algorithms can not only complement the many exponential lower bounds existing in this field (and by this indicate that the problem is well-understood), but they can also rule out that the algorithm does not suffer from an even worse runtime behavior such as the not uncommon $n^{\Theta(n)}$ runtime. 

With Theorem~\ref{thm:main} we provided a simple and general approach towards proving exponential upper bounds. With this method, we easily proved exponential upper bounds for various algorithmic settings.

In this first work on exponential-time evolutionary algorithms, we have surely not developed the full potential of this perspective in evolutionary computation. The clearly most important question for future work is what can be said about the constant $C$ in the $\poly(n) C^n$ runtime guarantee. A constant $C$ less than $2$ shows that the algorithm is superior to random or exhaustive search. Taking again the field of classic algorithms as example, another interesting question is if there are EAs with ``nice'' exponential runtimes such as, e.g., the $1.0836^n$ runtime of the algorithm of Xiao and Nagamochi~\cite{XiaoN13} for finding maximum independent sets in graphs with maximum degree~3. 

We mention two directions in which extensions of our method would be desirable. In Section~\ref{sec:simpleGA}, we have analyzed the mutation-based simple GA, but as many previous works, we have shied away from the original simple GA using crossover. Our proof, regarding one lucky lineage, does not apply to settings with crossover. Hence nothing better than the trivial $n^n$ upper bound is known, which is quite far from the $2^{n^{\eps/11}}$ lower bound shown in~\cite{OlivetoW15} for $\mu \le n^{(1/4)-\eps}$. 

In Section~\ref{sec:appli}, we have applied our general method, formulated for single-trajectory algorithms, mostly to ${(1+1)}$-type algorithms. The reason is that when regarding more offspring, say in a run of the \oplea for $\lambda > 1$, then each of these could potentially interfere with the survival of the desired offspring. Our methods would remain applicable when also excluding such interferences in the lucky event we are waiting for, but this would lead to upper bounds of order $\exp(O(n\lambda))$. We do not believe that larger offspring populations are that detrimental and hope that stronger methods can prove upper bounds which are exponential only in the problem size~$n$.

\newcommand{\etalchar}[1]{$^{#1}$}


}
\end{document}